\documentclass[]{article}
\usepackage{algorithm}
\usepackage[noend]{algorithmic}
\usepackage{amsmath}
\usepackage{amssymb}
\usepackage{amsthm}
\usepackage{bbm}
\usepackage{bm}
\usepackage{color}
\usepackage{dsfont}
\usepackage{enumerate}
\usepackage{graphicx}
\usepackage{proceed2e}
\usepackage{subfigure}
\usepackage{times}

\usepackage[usenames,dvipsnames]{xcolor}
\usepackage[bookmarks=false]{hyperref}
\hypersetup{
  pdftex,
  pdffitwindow=true,
  pdfstartview={FitH},
  pdfnewwindow=true,
  colorlinks,
  linktocpage=true,
  linkcolor=Green,
  urlcolor=Green,
  citecolor=Green
}

\usepackage[backgroundcolor=White,textwidth=0.6in]{todonotes}
\usepackage[capitalize]{cleveref}

\newcommand{\commentout}[1]{}
\newcommand{\junk}[1]{}
\newcommand{\etal}{\emph{et al.}}

\newtheorem{theorem}{Theorem}

\newtheorem{lemma}{Lemma}

\newcommand{\cE}{\mathcal{E}}

\newcommand{\cG}{\mathcal{G}}
\newcommand{\cH}{\mathcal{H}}

\newcommand{\cN}{\mathcal{N}}

\newcommand{\realset}{\mathbb{R}}

\newcommand{\abs}[1]{\left|#1\right|}

\newcommand{\E}[2]{\mathbb{E}_{#2} \! \left[#1\right]}
\newcommand{\EE}[1]{\mathbb{E} \left[#1\right]}

\newcommand{\I}[1]{\mathds{1} \! \left\{#1\right\}}

\newcommand{\rnd}[1]{\mathbf{#1}}
\newcommand{\set}[1]{\left\{#1\right\}}

\newcommand{\transpose}{^\mathsf{\scriptscriptstyle T}}

\DeclareMathOperator*{\argmax}{arg\,max\,}

\mathchardef\mhyphen="2D

\newcommand{\cascadeklucb}{{\tt CascadeKL\mhyphen UCB}}
\newcommand{\cascadelints}{{\tt CascadeLinTS}}
\newcommand{\cascadelinucb}{{\tt CascadeLinUCB}}
\newcommand{\cascadeucb}{{\tt CascadeUCB1}}

\newcommand{\lints}{{\tt LinTS}}

\newcommand{\rankedlints}{{\tt RankedLinTS}}

\newcommand{\tr}{\mathrm{trace}}

\begin{document}

\title{Cascading Bandits for Large-Scale Recommendation Problems}

\author{Shi Zong \\
Dept of Electrical and Computer Engineering \\
Carnegie Mellon University \\
\emph{szong@andrew.cmu.edu} \And
Hao Ni \\
Dept of Electrical and Computer Engineering \\
Carnegie Mellon University \\
\emph{haon@cmu.edu} \AND
Kenny Sung \\
Dept of Electrical and Computer Engineering \\
Carnegie Mellon University \\
\emph{tsung@andrew.cmu.edu} \And
Nan Rosemary Ke \\
D\'{e}pt d'informatique et de recherche op\'{e}rationnelle \\
Universit\'{e} de Montr\'{e}al \\
\emph{nke001@gmail.com} \AND
Zheng Wen \\
Adobe Research \\
San Jose, CA \\
\emph{zwen@adobe.com} \And
Branislav Kveton \\
Adobe Research \\
San Jose, CA \\
\emph{kveton@adobe.com}}

\maketitle

\begin{abstract}
Most recommender systems recommend a list of items. The user examines the list, from the first item to the last, and often chooses the first attractive item and does not examine the rest. This type of user behavior can be modeled by the \emph{cascade model}. In this work, we study \emph{cascading bandits}, an online learning variant of the cascade model where the goal is to recommend $K$ most attractive items from a large set of $L$ candidate items. We propose two algorithms for solving this problem, which are based on the idea of linear generalization. The key idea in our solutions is that we learn a predictor of the attraction probabilities of items from their features, as opposing to learning the attraction probability of each item independently as in the existing work. This results in practical learning algorithms whose regret does not depend on the number of items $L$. We bound the regret of one algorithm and comprehensively evaluate the other on a range of recommendation problems. The algorithm performs well and outperforms all baselines.
\end{abstract}


\section{INTRODUCTION}
\label{sec:introduction}

Most recommender systems recommended a list of $K$ items, such as restaurants, songs, or movies. The user \emph{examines} the recommended list from the first item to the last, and typically clicks on the first item that \emph{attracts} the user. The \emph{cascade model} \cite{craswell08experimental} is a popular model to formulate this kind of user behavior. The items before the first clicked item are \emph{not attractive}, because the user examines these items but does not click on them. The items after the first attractive item are \emph{unobserved}, because the user never examines these items. The key assumption in the cascade model is that each item attracts the user independently of the other items. Under this assumption, the optimal solution in the cascade model, the list of $K$ items that maximizes the probability that the user finds an attractive item, are $K$ most attractive items. The cascade model is simple, intuitive, and surprisingly effective in explaining user behavior \cite{chuklin15click}.

In this paper, we study on an online learning variant of the cascade model, which is known as \emph{cascading bandits} \cite{kveton15cascading}. In this model, the learning agent does not know the preferences of the user over recommended items and the goal is to learn them by interacting with the user. At time $t$, the agent recommends to the user a list of $K$ items out of $L$ candidate items and observes the click of the user. If the user clicks on an item, the agent receives a reward of one. If the user does not click on any item, the agent receives a reward of zero. The performance of the learning agent is evaluated by its cumulative reward in $n$ steps, which is the total number of clicks in $n$ steps. The goal of the agent is to maximize it.

Kveton \etal~\cite{kveton15cascading} proposed two computationally and sample efficient algorithms for cascading bandits. They also proved a $\Omega(L - K)$ lower bound on the regret in cascading bandits, which shows that the regret grows linearly with the number of candidate items $L$. Therefore, cascading bandits are impractical for learning when $L$ is large. Unfortunately, this setting is common practice. For instance, consider the problem of learning a personalized recommender system for $K = 10$ movies from the ground set of $L = 100\text{k}$ movies. In this setting, each movie would have to be shown to the user at least once, which means at least $10\text{k}$ interactions with the recommender system, before the system starts behaving intelligently. Such a system would clearly be impractical. The main contribution of our work is that we propose \emph{linear cascading bandits}, an online learning framework that makes learning in cascading bandits practical at scale. The key step in our approach is that we assume that the attraction probabilities of items can be predicted from the features of items. Features are often available in practice or can be easily derived.

To the best of our knowledge, this is the first work that studies a top-$K$ recommender problem in the bandit setting with cascading feedback and context. Specifically, we make four contributions. First, we propose linear cascading bandits, a variant of cascading bandits where we make an additional assumption that the attraction probabilities of items are a linear function of the features of items. This assumption is the key step in designing a sample efficient learning algorithm for our problem. Second, we propose two computationally efficient learning algorithms, $\cascadelints$ and $\cascadelinucb$, which are motivated by \emph{Thompson sampling (TS)} \cite{thompson33likelihood,agrawal12analysis} and \emph{linear UCB} \cite{abbasi-yadkori11improved,wen15efficient}, We believe this is the first application of linear generalization in the cascade model under partial monitoring feedback. Third, we derive an upper bound on the regret of $\cascadelinucb$ and discuss why a similar upper bound should hold for $\cascadelints$. Finally, we evaluate $\cascadelints$ on a range of recommendation problems; in the domains of restaurant, music, and movie recommendations; and demonstrate that it performs well even when our modeling assumptions are violated.

Our paper is organized as follows. In \cref{sec:background}, we review the cascade model and cascading bandits. In \cref{sec:linear cascading bandits}, we present linear cascading bandits; propose $\cascadelints$ and $\cascadelinucb$; and bound the regret of $\cascadelinucb$. In \cref{sec:experiments}, we evaluate  $\cascadelints$ on several recommendation problems. We review related work in \cref{sec:related work} and conclude in \cref{sec:conclusions}.

To simplify exposition, we denote random variables by boldface letter. We define $[n] = \set{1, \dots, n}$ and denote the cardinality of set $A$ by $\abs{A}$.


\section{BACKGROUND}
\label{sec:background}

In this section, we review the cascade model \cite{craswell08experimental} and cascading bandits \cite{kveton15cascading}.

\subsection{Cascade Model}
\label{sec:cascade model}

The \emph{cascade model} \cite{craswell08experimental} is a popular model of user behavior. In this model, the user is recommended a list of $K$ items $A = (a_1, \dots, a_K) \in \Pi_K(E)$, where $\Pi_K(E)$ is the set of all \emph{$K$-permutations} of some \emph{ground set} $E = [L]$, which is the set of all possibly recommended items. The model is parameterized by $L$ \emph{attraction probabilities} $\bar{w} \in [0, 1]^E$ and the user scans the list $A$ sequentially from the first item $a_1$ to the last $a_K$. After the user examines item $a_k$, the item attracts the user with probability $\bar{w}(a_k)$, \emph{independently} of the other items. If the user is attracted by item $a_k$, the user clicks on it and stop examining the remaining items. If the user is not attracted by item $a_k$, the user examines the next recommended item $a_{k + 1}$. It is easy to see that the probability that item $a_k$ is examined is $\prod_{i = 1}^{k - 1} (1 - \bar{w}(a_i))$, and that the probability that at least one item in $A$ is attractive is $1 - \prod_{i = 1}^{K} (1 - \bar{w}({a_i}))$. This objective is maximized by $K$ most attractive items.

The cascade model is surprising effective in explaining how users scan lists of items \cite{chuklin15click}. The reason is that lower ranked items typically do not get clicked because the user is attracted by higher ranked items, and never examines the rest of the recommended list.

\subsection{Cascading Bandits}
\label{sec:cascading bandits}

Kveton \etal~\cite{kveton15cascading} proposed a learning variant of the cascading model, which is known as a cascading bandit. Formally, a \emph{cascading bandit} is a tuple $B = (E, P, K)$, where $E = [L]$ is a \emph{ground set} of $L$ items, $P$ is a probability distribution over a binary hypercube $\set{0, 1}^E$, and $K \leq L$ is the number of recommended items.

The learning agent interacts with our problem as follows. Let $(\rnd{w}_t)_{t = 1}^n$ be an i.i.d. sequence of $n$ \emph{weights} drawn from $P$, where $\rnd{w}_t \in \set{0, 1}^E$ and $\rnd{w}_t(e)$ is the preference of the user for item $e$ at time $t$. More precisely, $\rnd{w}_t(e) = 1$ if and only if item $e$ attracts the user at time $t$. At time $t$, the agent recommends a list of $K$ items $\rnd{A}_t = (\rnd{a}^t_1, \dots, \rnd{a}^t_K) \in \Pi_K(E)$. The list is a function of the observations of the agent up to time $t$. The user examines the list, from the first item $\rnd{a}^t_1$ to the last $\rnd{a}^t_K$, and clicks on the first attractive item. If the user is not attracted by any item, the user does not click on any item. Then time increases to $t + 1$.

The reward of the agent at time $t$ is one if and only if the user is attracted by at least one item in $\rnd{A}_t$. Formally, the reward at time $t$ can be expressed as $\rnd{r}_t = f(\rnd{A}_t, \rnd{w}_t)$, where $f: \Pi_K(E) \times [0, 1]^E \to [0, 1]$ is a \emph{reward function} and we define it as:
\begin{align*}
  f(A, w) = 1 - \prod_{k = 1}^K (1 - w(a_k))
\end{align*}
for any $A = (a_1, \dots, a_K) \in \Pi_K(E)$ and $w \in [0, 1]^E$. The agent at time $t$ receives feedback:
\begin{align*}
  \rnd{C}_t = \min \set{k \in [K]: \rnd{w}_t(\rnd{a}^t_k) = 1}\,,
\end{align*}
where we assume that $\min \emptyset = \infty$. The feedback $\rnd{C}_t$ is the click of the user. If $\rnd{C}_t \leq K$, the user clicks on item $\rnd{C}_t$. If $\rnd{C}_t = \infty$, the user does not click on any item. Since the user clicks on the first attractive item in the list, the observed weights of all recommended items at time $t$ can be expressed as a function of $\rnd{C}_t$:
\begin{align}
  \rnd{w}_t(\rnd{a}^t_k) = \I{\rnd{C}_t = k} \quad k = 1, \dots, \min \set{\rnd{C}_t, K}\,.
  \label{eq:click}
\end{align}
Accordingly, we say that item $e$ is \emph{observed} at time $t$ if $e = \rnd{a}^t_k$ for some $k \in [\min \set{\rnd{C}_t, K}]$.

Let the attraction weights of items in the ground set $E$ be distributed independently as:
\begin{align*}
  P(w) = \prod_{e \in E} \mathrm{Ber}(w(e); \bar{w}(e))\,,
\end{align*}
where $\mathrm{Ber}(\cdot; \theta)$ is a Bernoulli distribution with mean $\theta$. Then the expected reward for list $A \in \Pi_K(E)$, the probability that at least one item in $A$ is satisfactory, can be expressed as $\EE{f(A, \rnd{w})} = f(A, \bar{w})$, and depends only on the attraction probabilities of individual items in $A$. Therefore, it is sufficient to learn a good approximation to $\bar{w}$ to act optimally.

The agent's policy is evaluated by its \emph{expected cumulative regret}:
\begin{align}
  R(n) = \EE{\sum_{t = 1}^n R(\rnd{A}_t, \rnd{w}_t)}\,,
  \label{eq:expected cumulative regret}
\end{align}
where $R(\rnd{A}_t, \rnd{w}_t) = f(A^\ast, \rnd{w}_t) - f(\rnd{A}_t, \rnd{w}_t)$ is the \emph{instantaneous stochastic regret} of the agent at time $t$ and:
\begin{align*}
  A^\ast = \argmax_{A \in \Pi_K(E)} f(A, \bar{w})
\end{align*}
is the \emph{optimal list} of items, the list that maximizes the reward at any time $t$. For simplicity of exposition, we assume that the optimal solution, as a set, is unique.

\subsection{Algorithm $\cascadeucb$}
\label{sec:algorithm CascadeUCB1}

Kveton \etal~\cite{kveton15cascading} proposed and analyzed two learning algorithms for cascading bandits, $\cascadeucb$ and $\cascadeklucb$. In this section, we review $\cascadeucb$.

$\cascadeucb$ belongs to the family of UCB algorithms. The algorithm operates in three stages. First, it computes the \emph{upper confidence bounds (UCBs)} $\rnd{U}_t \in [0, 1]^E$ on the attraction probabilities of all items in $E$. The UCB of item $e$ at time $t$ is:
\begin{align}
  \rnd{U}_t(e) = \hat{\rnd{w}}_{\rnd{T}_{t - 1}(e)}(e) + c_{t - 1, \rnd{T}_{t - 1}(e)}\,,
  \label{eq:UCB}
\end{align}
where $\hat{\rnd{w}}_s(e)$ is the average of $s$ observed attraction weights of item $e$, $\rnd{T}_t(e)$ is the number of times that item $e$ is observed in $t$ steps, and:
\begin{align*}
  c_{t, s} = \sqrt{(1.5 \log t) / s}
\end{align*}
is the radius of a confidence interval around $\hat{\rnd{w}}_s(e)$ after $t$ steps such that $\bar{w}(e) \in [\hat{\rnd{w}}_s(e) - c_{t, s}, \hat{\rnd{w}}_s(e) + c_{t, s}]$ holds with high probability. Second, $\cascadeucb$ recommends a list of $K$ items with largest UCBs:
\begin{align*}
  \rnd{A}_t = \argmax_{A \in \Pi_K(E)} f(A, \rnd{U}_t)\,.
\end{align*}
Finally, after the user provides feedback $\rnd{C}_t$, the algorithm updates its estimates of the attraction probabilities $\bar{w}(e)$ based on the observed weights of items, which are defined in \eqref{eq:click} for all $e = \rnd{a}^t_k$ such that $k \leq \rnd{C}_t$.


\section{LINEAR CASCADING BANDITS}
\label{sec:linear cascading bandits}

Kveton \etal~\cite{kveton15cascading} showed that the $n$-step regret of $\cascadeucb$ is $O((L - K) (1 / \Delta) \log n)$, where $L$ is the number of items in ground set $E$; $K$ is the number of recommended items; and $\Delta$ is the gap, which measures the sample complexity. This means that the regret increases linearly with the number of items $L$. As a result, $\cascadeucb$ is not practical when $L$ is large. Unfortunately, this setting is common practice. For instance, consider the problem of learning a personalized recommender for $10$ movies from the ground set of $100\text{k}$ movies. To learn, $\cascadeucb$ would need to show each movie to the user at least once, which means that the algorithm would require at least $10$k interactions with the user to start behaving intelligently. This is clearly impractical.

In this work, we propose \emph{practical algorithms} for large-scale cascading bandits, in the setting where $L$ is large. The key assumption, which allows us to learn efficiently, is that we assume that the attraction probability of each item $e$, $\bar{w}(e)$, can be approximated by a linear combination of some known $d$-dimensional feature vector $x_e \in \realset^{d \times 1}$ and an unknown $d$-dimensional parameter vector of $\theta^\ast \in \realset^{d \times 1}$, which is shared among all items. More precisely, we assume that there exists $\theta^\ast \in \Theta$ such that:
\begin{align}
  \bar{w}(e) \approx x_e\transpose \theta^\ast
  \label{eq:linear generalization}
\end{align}
for any $e \in E$. The features are problem specific and we discuss how to construct them in \cref{sec:features}. We propose two learning algorithms, which we call \emph{cascading linear Thompson sampling} ($\cascadelints$) and \emph{cascading linear UCB} ($\cascadelinucb$). We prove that when the above linear generalization is perfect, the regret of $\cascadelinucb$ is independent of $L$ and sublinear in $n$. Therefore, $\cascadelinucb$ is suitable for learning to recommend from large ground sets $E$. We also discuss why a similar regret bound should hold for $\cascadelints$, though we do not prove this bound formally.

\subsection{Algorithms}
\label{sec:algorithm}

Our learning algorithms are based on the ideas of Thompson sampling \cite{thompson33likelihood,agrawal12analysis} and linear UCB \cite{abbasi-yadkori11improved}, and motivated by the recent work of Wen \etal~\cite{wen15efficient}, which proposes computationally and sample efficient algorithms for large-scale stochastic combinatorial semi-bandits. The pseudocode of both algorithms is in Algorithms~\ref{alg:linear TS} and \ref{alg:linear UCB}, and we outline them below. 

Both $\cascadelints$ and $\cascadelinucb$ represent their past observations as a positive-definite matrix $\rnd{M}_t \in \realset^{d \times d}$ and a vector $\rnd{B}_t \in \realset^{d \times 1}$. Specifically, let $\rnd{X}_t$ be a matrix whose rows are the  feature vectors of all observed items in $t$ steps and $\rnd{Y}_t$ be a column vector of all observed attraction weights in $t$ steps. Then:
\begin{align*}
  \rnd{M}_t = \sigma^{-2} \rnd{X}_t\transpose \rnd{X}_t + I_d
\end{align*}
is the \emph{gram matrix} in $t$ steps and:
\begin{align*}
  \rnd{B}_t = \rnd{X}_t\transpose \rnd{Y}_t\,,
\end{align*}
where $I_d$ is a $d \times d$ identity matrix and $\sigma > 0$ is parameter that controls the learning rate.\footnote{Ideally, $\sigma^2$ should be the variance of the observation noises. However, based on recent literature \cite{wen15efficient}, we believe that both algorithms will perform well for a wide range of $\sigma^2$.}

Both $\cascadelints$ and $\cascadelinucb$ operate in three stages. First, they estimated the expected weight of each item $e$ based on their model of the world. $\cascadelints$ randomly samples parameter vector $\theta_t$ from a normal distribution, which approximates its posterior on $\theta^\ast$, and then estimates the expected weight as $x_e\transpose \theta_t$. $\cascadelinucb$ computes an upper confidence bound $\rnd{U}_t(e)$ for each item $e$. Second, both algorithms choose the optimal list $\rnd{A}_t$ with respect to their estimates. Finally, they receive feedback, and update $\rnd{M}_t$ and $\rnd{B}_t$ using \cref{alg:linear update}.


\begin{algorithm}[t]
  \caption{$\cascadelints$}
  \label{alg:linear TS}
  \begin{algorithmic}
    \STATE \textbf{Inputs:} Variance $\sigma^2$
    \STATE
    \STATE // Initialization
    \STATE $\rnd{M}_0 \gets I_d$ and $\rnd{B}_0 \gets \mathbf{0}$
    \STATE
    \FORALL{$t = 1, \dots, n$}
      \STATE $\bar{\theta}_{t - 1} \gets \sigma^{-2} \rnd{M}_{t - 1}^{-1} \rnd{B}_{t - 1}$
      \STATE $\theta_t \sim \cN(\bar{\theta}_{t - 1}, \rnd{M}_{t - 1}^{-1})$
      \STATE
      \STATE // Recommend a list of $K$ items and get feedback
      \FORALL{$k = 1, \dots, K$}
        \STATE $\rnd{a}^t_k \gets \argmax_{e \in [L] - \{\rnd{a}^t_1, \dots, \rnd{a}^t_{k - 1}\}} x_e\transpose \theta_t$
      \ENDFOR
      \STATE $\rnd{A}_t \gets (\rnd{a}^t_1, \dots, \rnd{a}^t_K)$
      \STATE Observe click $\rnd{C}_t \in \set{1, \dots, K, \infty}$
      \STATE Update statistics using \cref{alg:linear update}
  \ENDFOR
  \end{algorithmic}
\end{algorithm}

\begin{algorithm}[t]
  \caption{$\cascadelinucb$}
  \label{alg:linear UCB}
  \begin{algorithmic}
    \STATE \textbf{Inputs:} Variance $\sigma^2$, constant $c$ (\cref{sec:analysis})
    \STATE
    \STATE // Initialization
    \STATE $\rnd{M}_0 \gets I_d$ and $\rnd{B}_0 \gets \mathbf{0}$
    \STATE 
    \FORALL{$t = 1, \dots, n$}
      \STATE $\bar{\theta}_{t - 1} \gets \sigma^{-2} \rnd{M}_{t - 1}^{-1} \rnd{B}_{t - 1}$
      \FORALL{$e \in E$}
        \STATE $\rnd{U}_t(e) \gets \min \set{x_e\transpose \bar{\theta}_{t - 1} +
        c \sqrt{x_e\transpose \rnd{M}_{t - 1}^{-1} x_e}, 1}$
      \ENDFOR
      \STATE
      \STATE // Recommend a list of $K$ items and get feedback
      \FORALL{$k = 1, \dots, K$}
        \STATE $\rnd{a}^t_k \gets \argmax_{e \in [L] - \{\rnd{a}^t_1, \dots, \rnd{a}^t_{k - 1}\}} \rnd{U}_t(e)$
      \ENDFOR
      \STATE $\rnd{A}_t \gets (\rnd{a}^t_1, \dots, \rnd{a}^t_K)$
      \STATE Observe click $\rnd{C}_t \in \set{1, \dots, K, \infty}$
      \STATE Update statistics using \cref{alg:linear update}
  \ENDFOR
  \end{algorithmic}
\end{algorithm}

\begin{algorithm}[t]
  \caption{Update of statistics in Algorithms~\ref{alg:linear TS} and \ref{alg:linear UCB}}
  \label{alg:linear update}
  \begin{algorithmic}
      \STATE $\rnd{M}_t \gets \rnd{M}_{t - 1}$
      \STATE $\rnd{B}_t \gets \rnd{B}_{t - 1}$
      \FORALL{$k = 1, \dots, \min \set{\rnd{C}_t, K}$}
        \STATE $e \gets \rnd{a}^t_k$
        \STATE $\rnd{M}_t \gets \rnd{M}_t + \sigma^{-2} x_e x_e\transpose$
        \STATE $\rnd{B}_t \gets \rnd{B}_t + x_e \I{\rnd{C}_t = k}$
      \ENDFOR
  \end{algorithmic}
 \end{algorithm}

We would like to emphasize that both $\cascadelints$ and $\cascadelinucb$ are computationally efficient. In practice, we would update $\rnd{M}_t^{-1}$ instead of $\rnd{M}_t$. In particular, note that:
\begin{align*}
  \rnd{M}_t \gets \rnd{M}_t + \sigma^{-2} x_e x_e\transpose
\end{align*}
can be equivalently updated as:
\begin{align*}
  \rnd{M}_t^{-1} \gets
  \rnd{M}_t^{-1} - \frac{\rnd{M}_t^{-1} x_e x_e\transpose \rnd{M}_t^{-1}}{x_e\transpose \rnd{M}_t^{-1} x_e + \sigma^2}\,,
\end{align*}
and hence $\rnd{M}_t^{-1}$ can be updated incrementally and computationally efficiently in $O(d^2)$ time. It is easy to to see that the per-step time complexities of both $\cascadelints$ and $\cascadelinucb$ are $O(L (d^2 + K))$.



\subsection{Analysis and Discussion}
\label{sec:analysis}

We first derive a regret bound on $\cascadelinucb$, under the assumptions that (1) $\bar{w}(e) = x_e\transpose \theta^\ast$ for all $e \in E$ and (2) $\| x_e \|_2 \leq 1$ for all $e \in E$. Note that condition (2) can be always ensured by rescaling feature vectors. The regret bound is detailed below.

\begin{theorem}
\label{thm:bound}
Under the above assumptions, for any $\sigma > 0$ and any
\[
c \geq  \frac{1}{\sigma}  \sqrt{d \log \left( 1 + \frac{nK}{d \sigma^2} \right)+ 2 \log \left ( nK \right) } + \|\theta^*
\|_2 ,
\]
if we run $\cascadelinucb$ with parameters $\sigma$ and $c$, then
\[
R(n) \leq 2c K \sqrt{
\frac{dn \log 
 \left[
 1+  \frac{nK}{d\sigma^2}
 \right]}{\log \left( 1 + \frac{1}{\sigma^2}   \right)} }+ 1.
\]
\end{theorem}
Note that if we choose $\sigma=1$ and
\[
c = \sqrt{d \log \left( 1 + \frac{nK}{d} \right)+ 2 \log \left ( nK \right) } + \eta,
\]
for some constant $\eta \geq \| \theta^\ast\|_2$, then $R (n) \leq \tilde{O} \left( K d \sqrt{n}\right)$ where the $\tilde{O}$ notation hides logarithmic factors.

The proof is in Appendix and we outline it below. First, we define event $\cG_{t, k} = \set{\text{item $\rnd{a}^t_k$ is examined in step $t$}}$ for any time $t$ and $k \in [K]$, and bound the $n$-step regret as
\begin{align*}
  R(n) \leq \EE{\sum_{t = 1}^n \sum_{k = 1}^K \I{\cG_{t, k}} [\bar{w}(\rnd{a}^{*, t}_k) - \bar{w}(\rnd{a}^t_k)]}\,,
\end{align*}
where $\rnd{a}^{*,t}_k$ is an optimal item in $A^\ast$ matched to item $\rnd{a}^t_k$ in step $t$. Second, we define an event
  \[ \cE=\left \{ \left | x_e^T (\bar{\theta}_{t-1} - \theta^*) \right | \leq c \| x_e\|_{\rnd{M}_{t - 1}^{-1}}\, \forall t \leq n, \, \forall e \in E\right \}, \]
where $\| x_e\|_{\rnd{M}_{t - 1}^{-1}} = \sqrt{x_e\transpose \rnd{M}_{t - 1}^{-1} x_e}$. Then we prove a high-probability bound $P(\cE) \geq 1-1/nK$ for any $c$ that satisfies the condition of \cref{thm:bound}. Finally, we show that by conditioning on $\cE$, we have
\begin{align*}
 & \sum_{t = 1}^n \sum_{k = 1}^K \I{\cG_{t, k}} [\bar{w}(\rnd{a}^{*, t}_k) - \bar{w}(\rnd{a}^t_k)] \\
 & \quad \leq 2 c \sum_{t = 1}^n \sum_{k = 1}^K \I{\cG_{t, k}} \|x_{\rnd{a}^t_k}\|_{\rnd{M}_{t - 1}^{-1}} \\
 & \quad \leq 2 c K \sqrt{\frac{d n \log\left[1 + \frac{n K}{d \sigma^2}\right]}{\log\left(1 + \frac{1}{\sigma^2}\right)}}\,,
\end{align*}
where the first inequality follows from the definition of $\cE$ and the second inequality follows from a worst-case bound.
The bound in \cref{thm:bound} follows from putting the above results together.

Recent work \cite{russo14learning,wen15efficient} demonstrated close relationships between UCB-like algorithms and Thompson sampling algorithms in related bandit problems. Therefore, we believe that a similar regret bound to that in \cref{thm:bound} also holds for $\cascadelints$. 
However, it is highly non-trivial to derive a regret bound for $\cascadelints$. Unlike in \cite{wen15efficient}, $\cascadelints$ cannot be analyzed from the Bayesian perspective because the Gaussian posterior is inconsistent with the fact that $\bar{w}(e)$ is bounded in $[0, 1]$. 
Moreover, a subtle statistical dependence between partial monitoring and Thompson sampling prevents a frequentist analysis similar to 
that in \cite{agrawal13thompson}. Therefore, we leave the formal analysis of $\cascadelints$ for future work. It is well known that Thompson sampling tends to outperform UCB-like algorithms in practice \cite{agrawal12analysis}. Therefore, we only empirically evaluate $\cascadelints$.


\section{EXPERIMENTS}
\label{sec:experiments}

We validate $\cascadelints$ on several problems of various sizes and from various domains. In each problem, we conduct several experiments that demonstrate that our approach is scalable and stable with respect to its tunable parameters, the number of recommended items $K$ and the number of features $d$.

Our experimental section is organized as follows. In \cref{sec:experimental setting}, we outline the experiments that are conducted on each dataset. In \cref{sec:metrics and baselines}, we introduce our metrics and baselines. In \cref{sec:features}, we describe how we construct the features of items $E$. We present our empirical results in the rest of the section.

\subsection{Experimental Setting}
\label{sec:experimental setting}

All of our learning problems can be viewed as follows. The feedback of users is a matrix $W \in \set{0, 1}^{m \times L}$, where row $i$ corresponds to user $i \in [m]$ and column $j$ corresponds to item $j \in E$. Entry $(i, j)$ of $W$, $W_{i, j} \in \set{0, 1}$, indicates that user $i$ is attracted by item $j$. The user at time $t$, the row of $W$, is chosen at random from the pool of all users. Our goal is to learn the list of items $A^\ast$, the columns of $W$, that maximizes the probability that the user at time $t$ is attracted by at least one recommended item.

In each of our problems, we conduct a set of experiments. In the first experiment, we compare $\cascadelints$ to baselines (\cref{sec:metrics and baselines}) and also evaluate its scalability. We experiment with three variants of our problems: $L = 16$ items, $L = 256$ items, and the maximum possible value of $L$ in a given experiment. The number of recommended items is $K = 4$ and the number of features is $d = 20$.

In the second experiment, we show that the performance of $\cascadelints$ is robust with respect to the number of features $d$, in the sense that $d$ affects the performance but $\cascadelints$ performs reasonably well for all settings of $d$. We experiment with three settings for the number of features: $d = 10$, $d = 20$, and $d = 40$. The ground set contains $L = 256$ items and the number of recommended items is $K = 4$.

In the third experiment, we evaluate $\cascadelints$ on an interesting subset of each dataset, such as \emph{Rock Songs}. The setting of this experiment is identical to the second experiment. This experiment validates that $\cascadelints$ can also learn to recommend items in the context, of a subset of the dataset.

In the last experiment, we evaluate how the performance of $\cascadelints$ varies with the number of recommended items $K$. We experiment with three settings for the number of recommended items: $K = 4$, $K = 8$, and $K = 12$. The ground set contains $L = 256$ items and the number of features is $d = 20$.

All experiments are conducted for $n = 100\text{k}$ steps and averaged over $10$ randomly initialized runs. The tunable parameter $\sigma$ in $\cascadelints$ is set to $1$.

\subsection{Metrics and Baselines}
\label{sec:metrics and baselines}

The performance of $\cascadelints$ is evaluated by its expected cumulative regret, which is defined in \eqref{eq:expected cumulative regret}. In most of our experiments, our modeling assumptions are violated. In particular, the items are not guaranteed to attract users independently because the attraction indicators $\rnd{w}_t(e)$ are correlated across items $e$. The result is that:
\begin{align*}
  A^\ast =
  \argmax_{A \in \Pi_K(E)} \EE{f(A, \rnd{w})} >
  \argmax_{A \in \Pi_K(E)} f(A, \bar{w})\,.
\end{align*}
It is NP-hard to find $A^\ast$, because $\EE{f(A, \rnd{w})}$ does not decompose into the product of expectations as we assume in our model (\cref{sec:cascading bandits}). However, since $\EE{f(A, \rnd{w})}$ is submodular and monotone in $A$, a $(1 - 1 / e)$ approximation to $A^\ast$ can be computed greedily, by iteratively adding items that attract most users that are not attracted by any previously added item. We denote this approximation by $A^\ast$ and use it instead of the optimal solution.

We compare $\cascadelints$ to two baselines. The first baseline is $\cascadeucb$ (\cref{sec:algorithm CascadeUCB1}). This baseline does not leverage the structure of our problem and learns the attraction probability of each item $e$ independently. The second baseline is $\rankedlints$ (\cref{alg:ranked linear TS}). This baseline is a variant of ranked bandits (\cref{sec:related work}), where the base bandit algorithm is $\lints$. This base algorithm is the same as in $\cascadelints$. Therefore, any observed difference in the performance of cascading and ranked bandits must be due to the efficiency of using the base algorithm, and not the algorithm itself. In this sense, our comparison of $\cascadelints$ and $\rankedlints$ is fair. The tunable parameter $\sigma$ in $\rankedlints$ is also set to $1$.

\subsection{Features}
\label{sec:features}

In most recommender problems, good features of items are rarely available. Thus, they are typically learned from data \cite{koren09matrix}. As an example, in movie recommendations, all state of the art approaches are based on collaborative filtering rather than on the features of movies, such as movie genres.

Motivated by the successes of collaborative filtering in recommender systems, we derive the features of our items using low-rank matrix factorization. In particular, let $W \in \set{0, 1}^{m \times L}$ be our feedback matrix for $m$ users and $L$ items. We randomly divide the rows of $W$ into two matrices, training matrix $W_\text{train} \in \set{0, 1}^{(m / 2) \times L}$ and test matrix $W_\text{test} \in \set{0, 1}^{(m / 2) \times L}$. We use $W_\text{train}$ to learn the features of items and $W_\text{test}$ in place of $W$ to evaluate our learning algorithms. Most existing real-world recommender systems already have some data about their users. Such data can be used to construct $W_\text{train}$.

Let $W_\text{train} \approx U \Sigma V\transpose$ be rank-$d$ truncated SVD of $W_\text{train}$, where $U \in \realset^{(m / 2) \times d}$, $\Sigma \in \realset^{d \times d}$, and $V \in \realset^{L \times d}$. Then the features of items are the rows of $V \Sigma$. Specifically, for each item $e \in E$ and feature $i \in [d]$, $x_e(i) = V_{e, i} \Sigma_{i, i}$.

\begin{algorithm}[t]
  \caption{Ranked bandits with linear TS.}
  \label{alg:ranked linear TS}
  \begin{algorithmic}
     \STATE \textbf{Inputs:} Variance $\sigma^2$
    \STATE
    \STATE // Initialization
    \STATE $\forall k \in [K]: \rnd{M}^k_0 \gets  I_d$ and $\rnd{B}^k_0 \gets \mathbf{0}$
    \STATE
    \FORALL{$t = 1, \dots, n$}
      \FORALL{$k = 1, \dots, K$}
        \STATE $\bar{\theta}^k_{t - 1} \gets \sigma^{-2} (\rnd{M}^k_{t - 1})^{-1} \rnd{B}^k_{t - 1}$
        \STATE $\theta^k_t \sim \cN(\bar{\theta}^k_{t - 1}, (\rnd{M}^k_{t - 1})^{-1})$
        \STATE $\rnd{a}^t_k \gets \argmax_{e \in [L] - \{\rnd{a}^t_1, \dots, \rnd{a}^t_{k - 1}\}} x_e\transpose \theta^k_t$
      \ENDFOR
      \STATE
      \STATE // Recommend a list of $K$ items and get feedback
      \STATE $\rnd{A}_t \gets (\rnd{a}^t_1, \dots, \rnd{a}^t_K)$
      \STATE Observe click $\rnd{C}_t \in \set{1, \dots, K, \infty}$
      \STATE
      \STATE // Update statistics
      \STATE $\forall k \in [K]: \rnd{M}^k_t \gets \rnd{M}^k_{t - 1}$
      \STATE $\forall k \in [K]: \rnd{B}^k_t \gets \rnd{B}^k_{t - 1}$
      \FORALL{$k = 1, \dots, \min \set{\rnd{C}_t, K}$}
        \STATE $e \gets \rnd{a}^t_k$
        \STATE $\rnd{M}^k_t \gets \rnd{M}^k_t + \sigma^{-2} x_e x_e\transpose$
        \STATE $\rnd{B}^k_t \gets \rnd{B}^k_t + x_e \I{\rnd{C}_t = k}$
      \ENDFOR
    \ENDFOR
  \end{algorithmic}
\end{algorithm}

\begin{figure*}[t]
\centering
\includegraphics[width=0.75\textwidth]{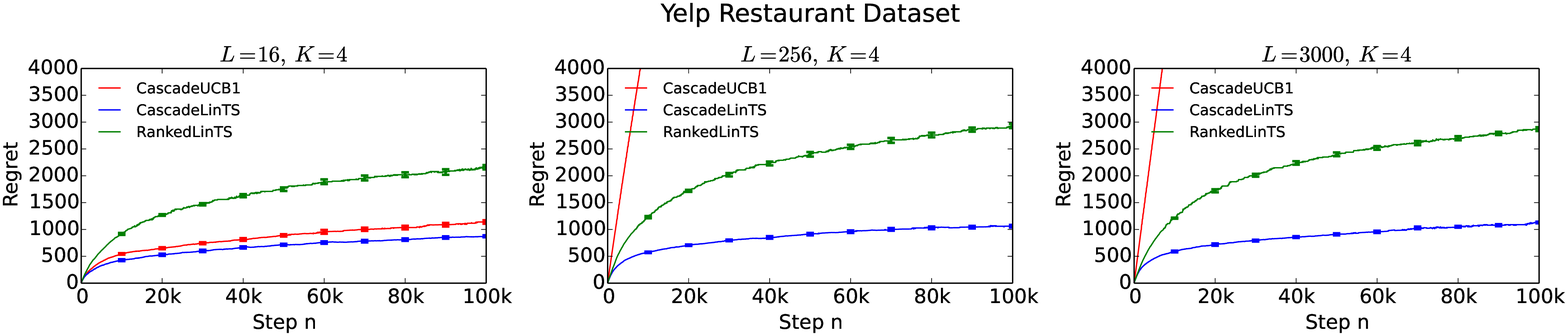}
\includegraphics[width=0.75\textwidth]{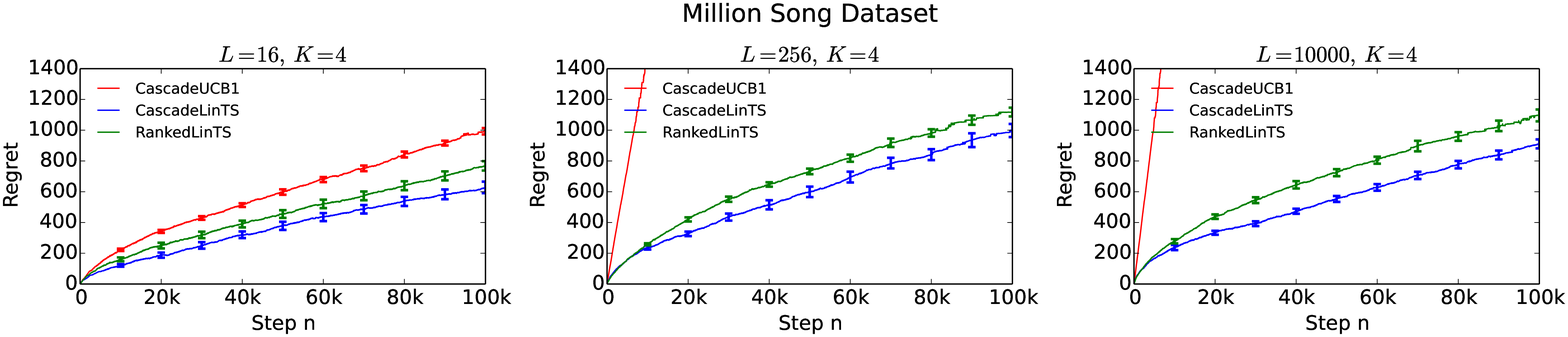}
\includegraphics[width=0.75\textwidth]{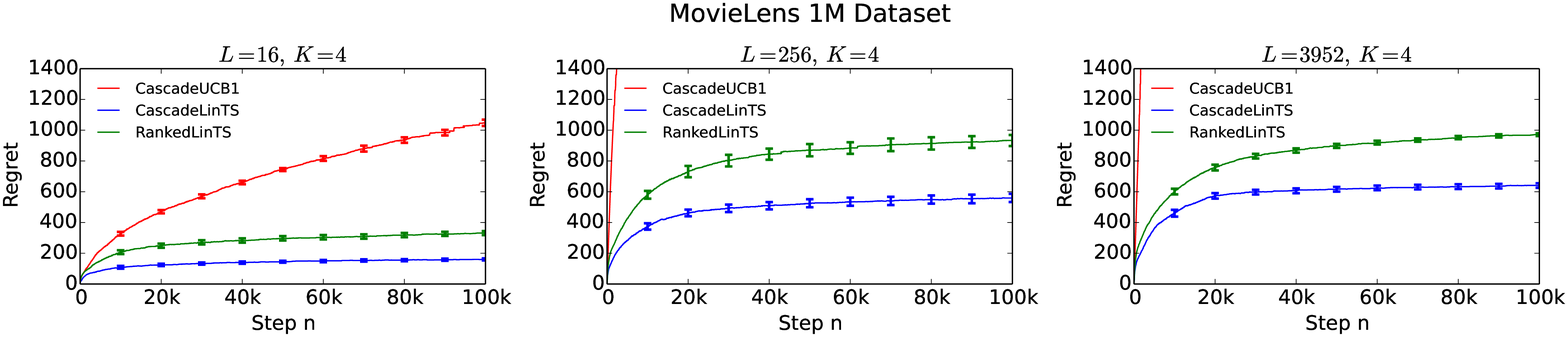}
\caption{The $n$-step regret of $\cascadeucb$, $\cascadelints$ and $\rankedlints$ on three problems. We vary the number of items in the ground set $E$, from $L = 16$ to the maximum value in each problem.}
\label{fig:exp1}
\end{figure*}

\begin{figure*}[t]
\centering
\includegraphics[width = \textwidth]{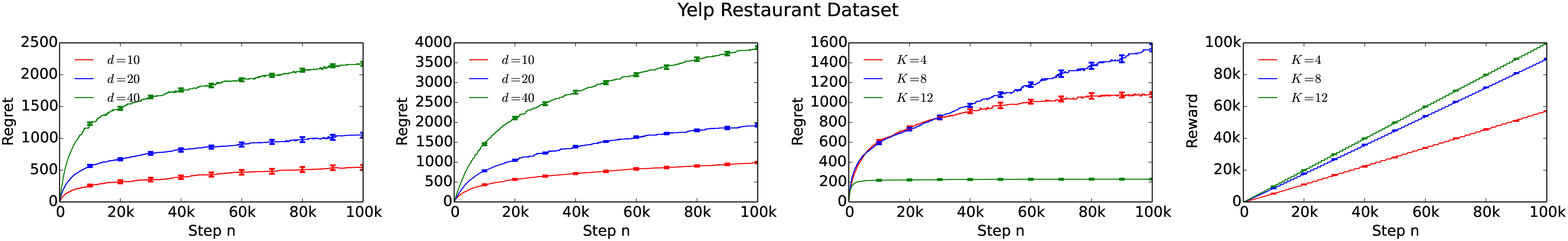}
\includegraphics[width = \textwidth]{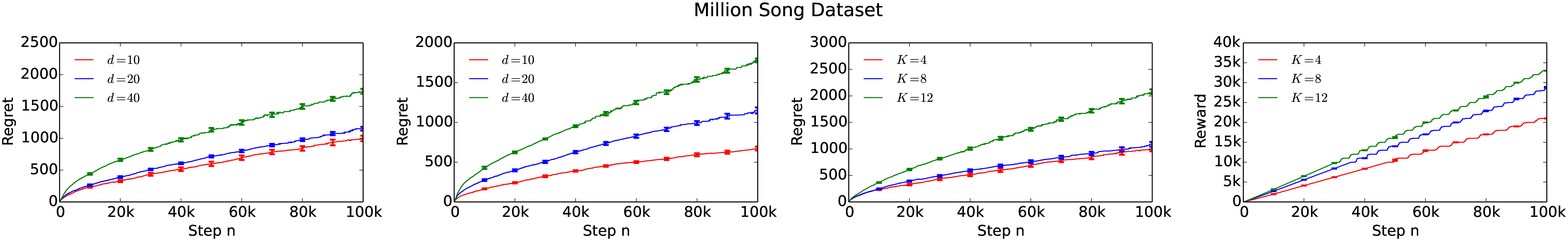}
\includegraphics[width = \textwidth]{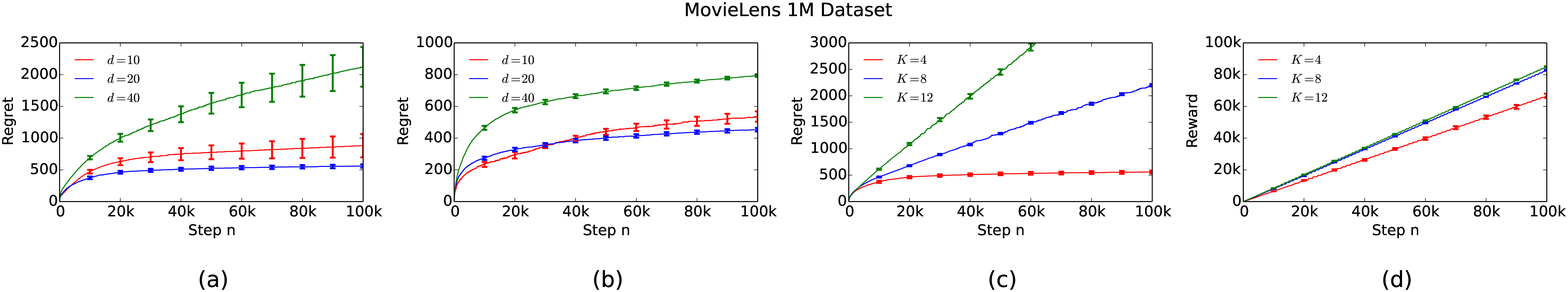}
\caption{{\bf a}. The $n$-step regret of $\cascadelints$ for varying number of features $d$. {\bf b}. The $n$-step regret of $\cascadelints$ in a subset of each dataset for varying number of features $d$. {\bf c}. The $n$-step regret of $\cascadelints$ for varying number of recommended items $K$. {\bf d}. The $n$-step reward of $\cascadelints$ for varying number of recommended items $K$.}
\label{fig:exp2-4}
\end{figure*}


\subsection{Restaurant Recommendations}
\label{sec:experiment hao}

Our dataset is from Yelp Dataset Challenge\footnote{\url{https://www.yelp.com/dataset_challenge}}. This dataset has five parts, including business information, checkin information, review information, tip information, and user information. We only consider the business and review information. The dataset contains $78\text{k}$ businesses, out of which $11\text{k}$ are restaurants; and $2.2\text{M}$ reviews written by $550\text{k}$ users. We extract $L = 3\text{k}$ most reviewed restaurants and $m = 20\text{k}$ most reviewing users.

Our objective is to maximize the probability that the user is attracted by at least one recommended restaurant. We build the model of users from past review data and assume that the user is attracted by the restaurant if the user reviewed this restaurant before. This indicates that the user visited the restaurant at some point in time, likely because the restaurant attracted the user at that time.

\subsubsection{Results}

The results of our first experiment are reported in \cref{fig:exp1}. When the ground set is small, $L = 16$, all compared methods perform similarly. In particular, the regret of $\cascadelints$ is similar to that of $\rankedlints$. The regret of $\cascadeucb$ is about two times larger than that of $\cascadelints$. As the size of the ground set increases, the gap between $\cascadelints$ and the other methods increases. In particular, when $L = 3\text{k}$, the regret of $\cascadeucb$ is orders of magnitude larger than that of $\cascadelints$, and the regret of $\rankedlints$ is almost three times larger.

In the second experiment (\cref{fig:exp2-4}a), we observe that $\cascadelints$ performs well for all settings of $d$. When the number of features doubles to $d = 40$, the regret roughly doubles. When the number of features is halved to $d = 10$, the regret improves and is roughly halved.

In the third experiment (\cref{fig:exp2-4}b), $\cascadelints$ is evaluated on the subset of \emph{American Restaurants}. This is the largest restaurant category in our dataset. We observe that $\cascadelints$ can learn for any number of features $d$, similarly to \cref{fig:exp2-4}a.

In the last experiment (\cref{fig:exp2-4}c), we observe that the regret of $\cascadelints$ increases with the number of recommended items, from $K = 4$ to $K = 8$. This result is surprising and seems to contradict to Kveton \etal~\cite{kveton15cascading}, who find both theoretically and empirically that the regret in cascading bandits decreases with the number of recommended items $K$. We investigate this further and plot the cumulative reward of $\cascadelints$ in \cref{fig:exp2-4}d. The reward increases with $K$, which is expected and validates that $\cascadelints$ learns better policies for larger $K$. Therefore, the increase in the regret in \cref{fig:exp2-4}c must be due to the fact that the expected reward of the optimal solution, $f(A^\ast, \bar{w})$, increases faster with $K$ than that of the learned policies. We believe that the optimal solutions for larger $K$ are harder to learn because our modeling assumptions are violated. In particular, the linear generalization in \eqref{eq:linear generalization} is imperfect and the items in $E$ are not guaranteed to attract users independently.

%


\subsection{Million Song Recommendation}
\label{sec:experiment kenny}

Million Song Dataset\footnote{\url{http://labrosa.ee.columbia.edu/millionsong/}} is a collection of audio features and metadata for a million contemporary pop songs. Instead of storing any audio, the dataset consists of features derived from the audio, user-song profile data, and genres of songs. 
We extract $L = 10\text{k}$ most popular songs from this dataset, as measured by the number of song-listening events; and $m = 400\text{k}$ most active users, as measured by the number of song-listening events.

Our objective is to maximize the probability that the user is attracted with at least one recommended song and plays it. We build the model of users from their past listening patterns and assume that the user is attracted by the song if the user listened to this song before. This indicates that the user was attracted by the song at some point in time.

\subsubsection{Results}

The results of our first experiment are reported in \cref{fig:exp1}. Similarly to \cref{sec:experiment hao}, we observe that when the ground set is small, $L = 16$, the regret of all compared methods is similar. As the size of the ground set increases, the gap between $\cascadeucb$ and the rest of the methods increases, and the regret of $\cascadeucb$ is orders of magnitude larger than that of $\cascadelints$. The regret of $\cascadelints$ is similar to that of $\rankedlints$ for all settings of $L$.

We report the regret of $\cascadelints$ for various numbers of features $d$, on the whole dataset and its subset of \emph{Rock Songs}, in \cref{fig:exp2-4}a and \ref{fig:exp2-4}b, respectively. Similarly to \cref{sec:experiment hao}, we observe that $\cascadelints$ performs well for all settings of $d$. The lowest regret in both experiments is achieved at $d = 10$.

In the last experiment (\cref{fig:exp2-4}c), we observe that the regret of $\cascadelints$ increases with the number of recommended items $K$. As in \cref{sec:experiment hao}, we observe that the cumulative reward of our learned policies increases with $K$. Therefore, the increase in the regret must be due to the fact that the expected reward of the optimal solution, $f(A^\ast, \bar{w})$, increases faster with $K$ than that of the learned policies. This is due to the mismatch between our model and real-world data.

\subsection{Movie Recommendation}
\label{sec:experiment shi}

MovieLens datasets\footnote{\url{http://grouplens.org/datasets/movielens/}} contain the ratings of users for movies from the MovieLens website. The datasets come in different sizes and we choose MovieLens 1M for our experiments. This dataset contains $1\text{M}$ anonymous ratings of $4\text{k}$ movies by $6\text{k}$ users who joined MovieLens in 2000.

We build the model of users from their historical ratings. The ratings are on a $5$-star scale and we assume the user is attracted by a movie if the user rates it with more than $3$ stars. Thus, the feedback matrix is defined as $W_{i, j} = \I{\text{user $i$ rates movie $j$ with more than $3$ stars}}$. Our goal is to maximize the probability of recommending at least one attractive movie.

\subsubsection{Results}

The results of our first experiment are reported in \cref{fig:exp1}. Similarly to \cref{sec:experiment hao}, we observe that the regret of all compared methods is similar when the ground set is small, $L$ = 16. The gap between $\cascadeucb$ and the rest of the methods increases when the size of the ground set increases. In particular, the regret of $\cascadeucb$ is orders of magnitude larger than that of $\cascadelints$. The regret of $\cascadelints$ is always lower than that of $\rankedlints$ for all settings of $L$.

We report the regret of $\cascadelints$ for various numbers of features $d$, on the whole dataset and its subset of \emph{Adventures}, in \cref{fig:exp2-4}a and \ref{fig:exp2-4}b, respectively. Similarly to Sections~\ref{sec:experiment hao} and \ref{sec:experiment kenny}, we observe that $\cascadelints$ performs well for all settings of $d$. The lowest regret in both experiments is achieved at $d = 20$.

In the last experiment (\cref{fig:exp2-4}c), we observe that the regret of $\cascadelints$ increases with the number of recommended items $K$. As in Sections~\ref{sec:experiment hao} and \ref{sec:experiment kenny}, the cumulative reward of our learned policies increases with $K$. Therefore, the increase in the regret must be due to the fact that the expected reward of the optimal solution, $f(A^\ast, \bar{w})$, increases faster with $K$ than that of the learned policies.

\label{sec:results shi}


\section{RELATED WORK}
\label{sec:related work}

Our work is closely related to \emph{cascading bandits} \cite{kveton15cascading,combes15learning}, which are learning variants of the cascade model of user behavior \cite{craswell08experimental}. The key difference is that we assume that the attraction weights of items are a linear function of known feature vectors, which are associated with each item; and an unknown parameter vector, which is learned. This leads to very efficient learning algorithms whose regret is sublinear in the number of items $L$. We compare $\cascadelints$ to $\cascadeucb$, one of the proposed algorithms by Kveton \etal~ \cite{kveton15cascading}, in \cref{sec:experiments}.

\emph{Ranked bandits} \cite{radlinski08learning} are a popular approach in learning to rank. The key idea in ranked bandits is to model each position in the recommended list as an independent bandit problem, which is then solved by a \emph{base bandit algorithm}. The solutions in ranked bandits are $(1 - 1 / e)$ approximate and their regret grows linearly with the number of recommended items $K$. On the other hand, ranked bandits do not assume that items attract the user independently. Slivkins \etal~\cite{slivkins13ranked} proposed contextual ranked bandits. We compare $\cascadelints$ to contextual ranked bandits with linear generalization in \cref{sec:experiments}.

Our learning problem is a partial monitoring problem where we do not observe the attraction weights of all recommended items. Bartok \etal~\cite{bartok12adaptive} studied general partial monitoring problems.
The algorithm of Bartok \etal~\cite{bartok12adaptive} scales at least linearly with the number of actions, which is $L \choose K$ in our setting. Therefore, the algorithm is impractical for large $L$ and moderate $K$. Agrawal \etal~\cite{agrawal89asymptotically} studied a variant of partial monitoring where the reward is observed. The algorithm of Agrawal \etal~\cite{agrawal89asymptotically} cannot be applied to our problem because the algorithm assumes a finite parameter set. Lin \etal~\cite{lin14combinatorial} and Kveton \etal~\cite{kveton15combinatorial} studied combinatorial partial monitoring. Our feedback model is similar to that of Kveton \etal~\cite{kveton15combinatorial}. Therefore, we believe that our algorithm and analysis can be relatively easily generalized to combinatorial action sets.

Our learning problem is combinatorial as we learn $K$ most attractive items out of $L$ candidate items. In this sense, our work is related to stochastic combinatorial bandits, which are frequently studied with a linear reward function and semi-bandit feedback \cite{gai12combinatorial,chen13combinatorial,kveton14matroid,kveton15tight,wen15efficient,combes15combinatorial}.  Our work differs from these approaches in both the reward function and feedback. Our reward function is a non-linear function of unknown parameters. Our feedback model is less than semi-bandit, because the learning agent does not observe the attraction weights of all recommended items.


\section{CONCLUSIONS}
\label{sec:conclusions}

In this work, we propose linear cascading bandits, a framework for learning to recommend in the cascade model at scale. The key assumption in linear cascading bandits is that the attraction probabilities of items are a linear function of the features of items, which are known; and an unknown parameter vector, which is unknown and we learn it. We design two algorithms for solving our problem, $\cascadelints$ and $\cascadelinucb$. We bound the regret of $\cascadelinucb$ and suggest that a similar regret bound can be proved for $\cascadelints$. We comprehensively evaluate $\cascadelints$ on a range of recommendation problems and compare it to several baselines. We report orders of magnitude improvements over learning algorithms that do not leverage the structure of our problem, the features of items. We observe empirically that $\cascadelints$ performs very well.

We leave open several questions of interest. For instance, we only bound the regret of $\cascadelinucb$. Based on the existing work \cite{wen15efficient}, we believe that a similar regret bound can be proved for $\cascadelints$. Moreover, note that our analysis of $\cascadelinucb$ is under the assumption that items attract the user independently and that the linear generalization is perfect. Both of these assumptions tend to be violated in practice. Our current analysis cannot explain this behavior and we leave it for future work.

The main limitation of the cascade model \cite{craswell08experimental} is that the user clicks on at most one item. This assumption is often violated in practice. Recently, Katariya \etal~\cite{katariya16dcm} proposed a generalization of cascading bandits to multiple clicks, by proposing a learning variant of the dependent click model \cite{guo09efficient}. We strongly believe that our results can be generalized to this setting and leave this for future work.

\bibliographystyle{plain}
\bibliography{References}

\begin{thebibliography}{10}

\bibitem{abbasi-yadkori11improved}
Yasin Abbasi-Yadkori, David Pal, and Csaba Szepesvari.
\newblock Improved algorithms for linear stochastic bandits.
\newblock In {\em Advances in Neural Information Processing Systems 24}, pages
  2312--2320, 2011.

\bibitem{agrawal89asymptotically}
Rajeev Agrawal, Demosthenis Teneketzis, and Venkatachalam Anantharam.
\newblock Asymptotically efficient adaptive allocation schemes for controlled
  i.i.d. processes: Finite parameter space.
\newblock {\em IEEE Transactions on Automatic Control}, 34(3):258--267, 1989.

\bibitem{agrawal12analysis}
Shipra Agrawal and Navin Goyal.
\newblock Analysis of {Thompson} sampling for the multi-armed bandit problem.
\newblock In {\em Proceeding of the 25th Annual Conference on Learning Theory},
  pages 39.1--39.26, 2012.

\bibitem{agrawal13thompson}
Shipra Agrawal and Navin Goyal.
\newblock Thompson sampling for contextual bandits with linear payoffs.
\newblock In {\em Proceedings of the 30th International Conference on Machine
  Learning}, pages 127--135, 2013.

\bibitem{bartok12adaptive}
Gabor Bartok, Navid Zolghadr, and Csaba Szepesvari.
\newblock An adaptive algorithm for finite stochastic partial monitoring.
\newblock In {\em Proceedings of the 29th International Conference on Machine
  Learning}, 2012.

\bibitem{chen13combinatorial}
Wei Chen, Yajun Wang, and Yang Yuan.
\newblock Combinatorial multi-armed bandit: General framework, results and
  applications.
\newblock In {\em Proceedings of the 30th International Conference on Machine
  Learning}, pages 151--159, 2013.

\bibitem{chuklin15click}
Aleksandr Chuklin, Ilya Markov, and Maarten de~Rijke.
\newblock {\em Click Models for Web Search}.
\newblock Morgan \& Claypool Publishers, 2015.

\bibitem{combes15learning}
Richard Combes, Stefan Magureanu, Alexandre Proutiere, and Cyrille Laroche.
\newblock Learning to rank: Regret lower bounds and efficient algorithms.
\newblock In {\em Proceedings of the 2015 ACM SIGMETRICS International
  Conference on Measurement and Modeling of Computer Systems}, 2015.

\bibitem{combes15combinatorial}
Richard Combes, Mohammad~Sadegh Talebi, Alexandre Proutiere, and Marc Lelarge.
\newblock Combinatorial bandits revisited.
\newblock In {\em Advances in Neural Information Processing Systems 28}, pages
  2107--2115, 2015.

\bibitem{craswell08experimental}
Nick Craswell, Onno Zoeter, Michael Taylor, and Bill Ramsey.
\newblock An experimental comparison of click position-bias models.
\newblock In {\em Proceedings of the 1st ACM International Conference on Web
  Search and Data Mining}, pages 87--94, 2008.

\bibitem{gai12combinatorial}
Yi~Gai, Bhaskar Krishnamachari, and Rahul Jain.
\newblock Combinatorial network optimization with unknown variables:
  Multi-armed bandits with linear rewards and individual observations.
\newblock {\em IEEE/ACM Transactions on Networking}, 20(5):1466--1478, 2012.

\bibitem{guo09efficient}
Fan Guo, Chao Liu, and Yi~Min Wang.
\newblock Efficient multiple-click models in web search.
\newblock In {\em Proceedings of the 2nd ACM International Conference on Web
  Search and Data Mining}, pages 124--131, 2009.

\bibitem{katariya16dcm}
Sumeet Katariya, Branislav Kveton, Csaba Szepesvari, and Zheng Wen.
\newblock {DCM} bandits: Learning to rank with multiple clicks.
\newblock In {\em Proceedings of the 33rd International Conference on Machine
  Learning}, 2016.

\bibitem{koren09matrix}
Yehuda Koren, Robert Bell, and Chris Volinsky.
\newblock Matrix factorization techniques for recommender systems.
\newblock {\em IEEE Computer}, 42(8):30--37, 2009.

\bibitem{kveton15cascading}
Branislav Kveton, Csaba Szepesvari, Zheng Wen, and Azin Ashkan.
\newblock Cascading bandits: Learning to rank in the cascade model.
\newblock In {\em Proceedings of the 32nd International Conference on Machine
  Learning}, 2015.

\bibitem{kveton14matroid}
Branislav Kveton, Zheng Wen, Azin Ashkan, Hoda Eydgahi, and Brian Eriksson.
\newblock Matroid bandits: Fast combinatorial optimization with learning.
\newblock In {\em Proceedings of the 30th Conference on Uncertainty in
  Artificial Intelligence}, pages 420--429, 2014.

\bibitem{kveton15combinatorial}
Branislav Kveton, Zheng Wen, Azin Ashkan, and Csaba Szepesvari.
\newblock Combinatorial cascading bandits.
\newblock In {\em Advances in Neural Information Processing Systems 28}, pages
  1450--1458, 2015.

\bibitem{kveton15tight}
Branislav Kveton, Zheng Wen, Azin Ashkan, and Csaba Szepesvari.
\newblock Tight regret bounds for stochastic combinatorial semi-bandits.
\newblock In {\em Proceedings of the 18th International Conference on
  Artificial Intelligence and Statistics}, 2015.

\bibitem{lin14combinatorial}
Tian Lin, Bruno Abrahao, Robert Kleinberg, John Lui, and Wei Chen.
\newblock Combinatorial partial monitoring game with linear feedback and its
  applications.
\newblock In {\em Proceedings of the 31st International Conference on Machine
  Learning}, pages 901--909, 2014.

\bibitem{radlinski08learning}
Filip Radlinski, Robert Kleinberg, and Thorsten Joachims.
\newblock Learning diverse rankings with multi-armed bandits.
\newblock In {\em Proceedings of the 25th International Conference on Machine
  Learning}, pages 784--791, 2008.

\bibitem{russo14learning}
Daniel Russo and Benjamin {Van Roy}.
\newblock Learning to optimize via posterior sampling.
\newblock {\em Mathematics of Operations Research}, 39(4):1221--1243, 2014.

\bibitem{slivkins13ranked}
Aleksandrs Slivkins, Filip Radlinski, and Sreenivas Gollapudi.
\newblock Ranked bandits in metric spaces: Learning diverse rankings over large
  document collections.
\newblock {\em Journal of Machine Learning Research}, 14(1):399--436, 2013.

\bibitem{thompson33likelihood}
William.~R. Thompson.
\newblock On the likelihood that one unknown probability exceeds another in
  view of the evidence of two samples.
\newblock {\em Biometrika}, 25(3-4):285--294, 1933.

\bibitem{wen15efficient}
Zheng Wen, Branislav Kveton, and Azin Ashkan.
\newblock Efficient learning in large-scale combinatorial semi-bandits.
\newblock In {\em Proceedings of the 32nd International Conference on Machine
  Learning}, 2015.

\end{thebibliography}

\clearpage
\onecolumn
\appendix


\newpage
\begin{center}
{\Large \bf Appendix}
\end{center}
\section{Proof for Theorem~\ref{thm:bound}}
\label{sec:analysis1}

\subsection{Notations}
We start by defining some notations. For each time $t$, we define a random permutation $(\rnd{a}^{*,t}_{1}, \dots, \rnd{a}^{*,t}_{K})$ of $A^*$ based on $\rnd{A}_t$ as follows:
for any $k=1,\dots, K$, if $\rnd{a}^t_k \in A^*$, then we set $\rnd{a}^{*,t}_k = \rnd{a}^t_k$. The remaining optimal items are positioned arbitrarily.
Notice that under this random permutation, we have:
\[
\bar{w} (\rnd{a}^{*,t}_k) \geq \bar{w} (\rnd{a}^{t}_k) \quad \text{and} \quad
\rnd{U}_t (\rnd{a}^{t}_k) \geq \rnd{U}_t (\rnd{a}^{*,t}_k) \quad \forall k=1,2,\ldots, K
\]

Moreover, we use $\cH_t$ to denote the ``history" (rigorously speaking, $\sigma$-algebra) by the end of time $t$. Then both
$\rnd{A}_t=(\rnd{a}^{t}_{1}, \dots, \rnd{a}^{t}_{K})$ and the permutation $(\rnd{a}^{*,t}_{1}, \dots, \rnd{a}^{*,t}_{K})$ of $A^*$ are
$\cH_{t-1}$-adaptive. In other words, they are conditionally deterministic at the beginning of time $t$. To simplify the notation, in this paper, we use $\E{\cdot}{t}$ to denote $\E{\cdot | \cH_{t-1}}{}$ when appropriate.

When appropriate, we also use $\langle \cdot, \cdot \rangle$ to denote the inner product of two vectors. Specifically, for two vectors $u$ and $v$
with the same dimension, we use $\langle u,v \rangle$ to denote $u\transpose v$.

\subsection{Regret Decomposition}

We first prove the following technical lemma:
\begin{lemma}
For any $B=(b_1, \dots, b_K) \in \Re^K$ and $C=(c_1, \dots, c_K) \in \Re^K$, we have
\[
\textstyle \prod_{k=1}^K b_k -  \prod_{k=1}^K c_k = \sum_{k=1}^K \left[ \prod_{i=1}^{k-1} b_i  \right] \times \left[ b_k - c_k \right] \times \left[  \prod_{j=k+1}^K c_j \right] .
\]
\label{lemma:tech1}
\end{lemma}
\begin{proof}
Notice that 
\begin{align}
 & \textstyle \sum_{k=1}^K \left[ \prod_{i=1}^{k-1} b_i  \right] \times \left[ b_k - c_k \right] \times \left[  \prod_{j=k+1}^K c_j \right] \nonumber \\
 = & \textstyle \sum_{k=1}^K \left \{  \left[ \prod_{i=1}^{k} b_i  \right] \times \left[  \prod_{j=k+1}^K c_j \right] -
 \left[ \prod_{i=1}^{k-1} b_i  \right] \times \left[  \prod_{j=k}^K c_j \right] \right \} \nonumber \\
 = & \textstyle \prod_{k=1}^K b_k - \prod_{k=1}^K c_k. \nonumber
\end{align}
\end{proof}
Thus we have
\begin{align}
R (\rnd{A}_t, \rnd{w}_t) = & f(A^*, \rnd{w}_t) - f (\rnd{A}_t, \rnd{w}_t) \nonumber \\
=&\textstyle  \prod_{k=1}^K \left( 1-\rnd{w}_t (\rnd{a}^t_k) \right) -  \prod_{k=1}^K \left( 1-\rnd{w}_t (\rnd{a}^{*,t}_k) \right) \nonumber \\
\stackrel{(a)}{=}& \textstyle \sum_{k=1}^K \left[ \prod_{i=1}^{k-1} \left( 1-\rnd{w}_t (\rnd{a}^t_i) \right) \right] \left[ \rnd{w}_t (\rnd{a}^{*,t}_k) - \rnd{w}_t (\rnd{a}^t_k) \right]
\left[ \prod_{j=k+1}^K  \left( 1-\rnd{w}_t (\rnd{a}^{*,t}_j) \right) \right] \nonumber \\
\stackrel{(b)}{\leq} & \textstyle \sum_{k=1}^K \left[ \prod_{i=1}^{k-1} \left( 1-\rnd{w}_t (\rnd{a}^t_i) \right) \right] \left[ \rnd{w}_t (\rnd{a}^{*,t}_k) - \rnd{w}_t (\rnd{a}^t_k) \right],
\end{align}
where equality (a) is based on Lemma~\ref{lemma:tech1} and inequality (b) is based on the fact that $\prod_{j=k+1}^K  \left( 1-\rnd{w}_t (\rnd{a}^{*,t}_j) \right) \leq 1$. Recall that $\rnd{A}^t$ and the permutation $(\rnd{a}^{*,t}_{1}, \dots, \rnd{a}^{*,t}_{K})$ of $A^*$ are deterministic conditioning on $\cH_{t-1}$, and $\rnd{a}_k^{*,t} \neq \rnd{a}_i^t$ for all $i<k$, thus we have
\begin{align}
\E{R (\rnd{A}_t, \rnd{w}_t)}{t} \leq & \, \textstyle \E { \sum_{k=1}^K \left[ \prod_{i=1}^{k-1} \left( 1-\rnd{w}_t (\rnd{a}^t_i) \right) \right] \left[ \rnd{w}_t (\rnd{a}^{*,t}_k) - \rnd{w}_t (\rnd{a}^t_k) \right]}{t} \nonumber \\
= & \, \textstyle
 \sum_{k=1}^K \E {
 \prod_{i=1}^{k-1} \left( 1-\rnd{w}_t (\rnd{a}^t_i) \right) }{t}  
 \E { \rnd{w}_t (\rnd{a}^{*,t}_k) - \rnd{w}_t (\rnd{a}^t_k) }{t} \nonumber \\
 =& \, \textstyle
  \sum_{k=1}^K \E {
 \prod_{i=1}^{k-1} \left( 1-\rnd{w}_t (\rnd{a}^t_i) \right) }{t}  
  \left[ \bar{w} (\rnd{a}^{*,t}_k) - \bar{w} (\rnd{a}^t_k) \right]. \nonumber
 \end{align}
 For any $t \leq n$ and any $e \in E$, we define event
 \[
 \cG_{t,k}=\left \{ \text{item $\rnd{a}^t_k$ is examined in episode $t$}\right \},
 \]
 notice that $\I{\cG_{t,k}}{}=\prod_{i=1}^{k-1} \left( 1-\rnd{w}_t (\rnd{a}^t_i) \right)$. Thus, we have
 \[
 \textstyle \E{\rnd{R}_t}{t} \leq \sum_{k=1}^K \E { \I{\cG_{t,k}}{}
  }{t}  
  \left[ \bar{w} (\rnd{a}^{*,t}_k) - \bar{w} (\rnd{a}^t_k) \right].
  \]
  Hence, from the tower property, we have
  \begin{align}
  \textstyle 
  R(n) \leq \E{
  \sum_{t=1}^n \sum_{k=1}^K   \I{\cG_{t,k}}{}  \left[ \bar{w} (\rnd{a}^{*,t}_k) - \bar{w} (\rnd{a}^t_k) \right]
  }{}. 
  \end{align}
  We further define event $\cE$ as
  \begin{align}
  \cE = \left \{ 
   \left | \langle x_e, \bar{\theta}_{t-1} - \theta^* \rangle  \right | 
  \leq 
 c \sqrt{x_e^T M_{t-1}^{-1} x_e} 
, \, \forall e \in E, \, \forall t \leq n \right \},
  \end{align}
  and $\bar{\cE}$ as the complement of $\cE$.
  Then we have
  \begin{align}
  R(n) \stackrel{(a)}{\leq} &  \textstyle P(\cE) \E{
  \sum_{t=1}^n \sum_{k=1}^K   \I{\cG_{t,k}}{}  \left[ \bar{w} (\rnd{a}^{*,t}_k) - \bar{w} (\rnd{a}^t_k) \right]
  \middle | \cE}{} \nonumber \\
 + &   \textstyle
  P(\bar{\cE}) \E{
  \sum_{t=1}^n \sum_{k=1}^K   \I{\cG_{t,k}}{}  \left[ \bar{w} (\rnd{a}^{*,t}_k) - \bar{w} (\rnd{a}^t_k) \right]
  \middle | \bar{\cE}}{} \nonumber \\
  \stackrel{(b)}{\leq} &  \textstyle
  \E{
  \sum_{t=1}^n \sum_{k=1}^K   \I{\cG_{t,k}}{}  \left[ \bar{w} (\rnd{a}^{*,t}_k) - \bar{w} (\rnd{a}^t_k) \right]
  \middle | \cE}{} + nK  P(\bar{\cE}), 
  \end{align}
  where inequality (a) is based on the law of total probability, and the inequality (b) is based on the naive bounds
  (1) $P(\cE) \leq 1$ and (2) $ \I{\cG_{t,k}}{}  \left[ \bar{w} (\rnd{a}^{*,t}_k) - \bar{w} (\rnd{a}^t_k) \right] \leq 1$.
  Notice that from the definition of event $\cE$, we have
  \[
  \bar{w}(e) = \langle x_e, \theta^*  \rangle  \leq  \langle x_e, \bar{\theta}_{t-1}  \rangle + c \sqrt{x_e^T M_{t-1}^{-1} x_e} 
  \quad \forall e \in E, \, \forall t \leq n
  \]
  under event $\cE$.
  Moreover, since $\bar{w}(e) \leq 1$ by definition, we have $\bar{w}(e) \leq \rnd{U}_t (e)$ for all $e \in E$ and all $t \leq n$ under event $\cE$. 
  Hence under event $\cE$, we have
  \[
   \bar{w} (\rnd{a}^t_k)  \leq \bar{w} (\rnd{a}^{*,t}_k)  \leq \rnd{U}_t (\rnd{a}^{*,t}_k) \leq \rnd{U}_t (\rnd{a}^{t}_k) \leq 
   \langle x_{\rnd{a}^{t}_k}, \bar{\theta}_{t-1}  \rangle + c \sqrt{x_{\rnd{a}^{t}_k}^T M_{t-1}^{-1} x_{\rnd{a}^{t}_k}} \quad
   \forall t \leq n.
  \]
  Thus we have
  \begin{align}
  \bar{w} (\rnd{a}^{*,t}_k)  -  \bar{w} (\rnd{a}^t_k) 
  \stackrel{(a)}{\leq} &  \langle x_{\rnd{a}^{t}_k}, \bar{\theta}_{t-1} - \theta^*  \rangle + c \sqrt{x_{\rnd{a}^{t}_k}^T M_{t-1}^{-1} x_{\rnd{a}^{t}_k}} \nonumber \\
   \stackrel{(b)}{\leq} &  2 c \sqrt{x_{\rnd{a}^{t}_k}^T M_{t-1}^{-1} x_{\rnd{a}^{t}_k}}, \nonumber
  \end{align}
  where inequality (a) follows from the fact that 
 $\bar{w} (\rnd{a}^{*,t}_k)  \leq \langle x_{\rnd{a}^{t}_k}, \bar{\theta}_{t-1}  \rangle + c \sqrt{x_{\rnd{a}^{t}_k}^T M_{t-1}^{-1} x_{\rnd{a}^{t}_k}}$ and inequality (b) follows from the fact that 
 $ \langle x_{\rnd{a}^{t}_k}, \bar{\theta}_{t-1} - \theta^*  \rangle \leq c \sqrt{x_{\rnd{a}^{t}_k}^T M_{t-1}^{-1} x_{\rnd{a}^{t}_k}}$ under event
 $\cE$. Thus, we have
 \begin{equation}
\textstyle R(n) \leq 2c \E{
 \sum_{t=1}^n \sum_{k=1}^K   \I{\cG_{t,k}}{}   \sqrt{x_{\rnd{a}^{t}_k}^T M_{t-1}^{-1} x_{\rnd{a}^{t}_k}} 
  \middle | \cE}{} + nK  P(\bar{\cE}). \nonumber
 \end{equation}
 Define $\rnd{K}_t = \min \{ \rnd{C}_t, K\}$, notice that 
 \[ \textstyle \sum_{k=1}^K   \I{\cG_{t,k}}{}   \sqrt{x_{\rnd{a}^{t}_k}^T M_{t-1}^{-1} x_{\rnd{a}^{t}_k}} =
 \sum_{k=1}^{\rnd{K}_t} \sqrt{x_{\rnd{a}^{t}_k}^T M_{t-1}^{-1} x_{\rnd{a}^{t}_k}} .
 \]
 Thus, we have
  \begin{equation}
 \textstyle R(n) \leq 2c \E{
 \sum_{t=1}^n \sum_{k=1}^{\rnd{K}_t}     \sqrt{x_{\rnd{a}^{t}_k}^T M_{t-1}^{-1} x_{\rnd{a}^{t}_k}} 
  \middle | \cE}{} + nK  P(\bar{\cE}). 
  \label{eqn:bound1}
 \end{equation}
 In the next two subsections, we will provide a \emph{worst-case} bound on $ \sum_{t=1}^n \sum_{k=1}^{\rnd{K}_t}    \sqrt{x_{\rnd{a}^{t}_k}^T M_{t-1}^{-1} x_{\rnd{a}^{t}_k}} $ and a bound on $P(\bar{\cE})$.

\subsection{Worst-Case Bound on $ \sum_{t=1}^n \sum_{k=1}^{\rnd{K}_t}  \sqrt{x_{\rnd{a}^{t}_k}^T M_{t-1}^{-1} x_{\rnd{a}^{t}_k}} $}

\begin{lemma}
$\sum_{t=1}^n \sum_{k=1}^{\rnd{K}_t}    \sqrt{x_{\rnd{a}^{t}_k}^T M_{t-1}^{-1} x_{\rnd{a}^{t}_k}} \leq 
K \sqrt{
\frac{dn \log 
 \left[
 1+  \frac{nK}{d\sigma^2}
 \right]}{\log \left( 1 + \frac{1}{\sigma^2}   \right)}
}$.
\label{lemma:tech2}
\end{lemma}
\begin{proof}
To simplify the exposition, we define $z_{t,k}=\sqrt{x_{\rnd{a}^{t}_k}^T M_{t-1}^{-1} x_{\rnd{a}^{t}_k}}$ 
for all $(t,k)$ s.t. $k \leq \rnd{K}_t$. Recall that
\[
M_t = M_{t-1} + \frac{1}{\sigma^2} \sum_{k=1}^{\rnd{K}_t} x_{\rnd{a}^{t}_k}  x_{\rnd{a}^{t}_k}^T
\]
Thus, for all $(t,k)$ s.t. $k \leq \rnd{K}_t$, we have that
\begin{align}
\det \left[M_t \right] \geq & \det \left[ M_{t-1} + \frac{1}{\sigma^2} x_{\rnd{a}^{t}_k}  x_{\rnd{a}^{t}_k}^T \right]
= \det \left[ M_{t-1}^{\frac{1}{2}} \left ( I  + \frac{1}{\sigma^2} M_{t-1}^{- \frac{1}{2}} x_{\rnd{a}^{t}_k}  x_{\rnd{a}^{t}_k}^T  M_{t-1}^{- \frac{1}{2}} \right) M_{t-1}^{\frac{1}{2}}  \right] \nonumber \\
= & \det \left[ M_{t-1} \right] \det \left[I  + \frac{1}{\sigma^2} M_{t-1}^{- \frac{1}{2}} x_{\rnd{a}^{t}_k}  x_{\rnd{a}^{t}_k}^T  M_{t-1}^{- \frac{1}{2}} \right] \nonumber \\
= &\det \left[ M_{t-1} \right] \left( 1 + \frac{1}{\sigma^2}  x_{\rnd{a}^{t}_k}^T M_{t-1}^{- 1}  x_{\rnd{a}^{t}_k} \right)
= \det \left[ M_{t-1} \right]   \left( 1 + \frac{z_{t,k}^2}{\sigma^2}   \right). \nonumber
\end{align}
Thus, we have
\[
\left( \det \left[M_t \right] \right)^{\rnd{K}_t} \geq \left( \det \left[ M_{t-1} \right] \right)^{\rnd{K}_t}  \prod_{k=1}^{\rnd{K}_t } \left( 1 + \frac{z_{t,k}^2}{\sigma^2}   \right).
\]
Since $\det \left[ M_t \right] \geq \det \left[ M_{t-1}  \right]$ and $\rnd{K}_t \leq K$, we have
\[
\left( \det \left[M_t \right] \right)^{K} \geq \left( \det \left[ M_{t-1} \right] \right)^{K}  \prod_{k=1}^{\rnd{K}_t } \left( 1 + \frac{z_{t,k}^2}{\sigma^2}   \right).
\]
So we have
\[
\left( \det \left[M_n \right] \right)^{K} \geq \left( \det \left[ M_{0} \right] \right)^{K} \prod_{t=1}^n \prod_{k=1}^{\rnd{K}_t } \left( 1 + \frac{z_{t,k}^2}{\sigma^2}   \right) = \prod_{t=1}^n \prod_{k=1}^{\rnd{K}_t } \left( 1 + \frac{z_{t,k}^2}{\sigma^2}   \right),
 \]
 since $M_0=I$. On the other hand, we have that
 \[
 \tr\left( M_n \right)=\tr \left(  I+ \frac{1}{\sigma^2} \sum_{t=1}^n \sum_{k=1}^{\rnd{K}_t} x_{\rnd{a}^{t}_k}  x_{\rnd{a}^{t}_k}^T \right)
 = d +  \frac{1}{\sigma^2} \sum_{t=1}^n \sum_{k=1}^{\rnd{K}_t} \| x_{\rnd{a}^{t}_k} \|_2^2 \leq d + \frac{nK}{\sigma^2},
 \]
 where the last inequality follows from the fact that $\| x_{\rnd{a}^{t}_k} \|_2 \leq 1$ and $\rnd{K}_t \leq K$.
 From the trace-determinant inequality, we have $\frac{1}{d}  \tr\left( M_n \right) \geq \left[ \det(M_n)\right]^{\frac{1}{d}}$,
 thus we have
 \[
 \left[
 1+  \frac{nK}{d \sigma^2}
 \right]^{dK}
 \geq 
 \left [
 \frac{1}{d}  \tr\left( M_n \right) 
 \right ]^{dK}
 \geq
 \left[ \det( M_n)\right]^K
 \geq
  \prod_{t=1}^n \prod_{k=1}^{\rnd{K}_t } \left( 1 + \frac{z_{t,k}^2}{\sigma^2}   \right).
 \]
Taking the logarithm, we have
\begin{equation}
d K \log 
 \left[
 1+  \frac{nK}{d \sigma^2}
 \right]
 \geq
 \sum_{t=1}^n \sum_{k=1}^{\rnd{K}_t} \log \left( 1 + \frac{z_{t,k}^2}{\sigma^2}   \right).
\end{equation}
Notice that $z_{t,k}^2=x_{\rnd{a}^{t}_k}^T M_{t-1}^{-1} x_{\rnd{a}^{t}_k} \leq x_{\rnd{a}^{t}_k}^T M_{0}^{-1} x_{\rnd{a}^{t}_k}
= \| x_{\rnd{a}^{t}_k} \|^2_2  \leq 1$, thus we have 
$z_{t,k}^2 \leq \frac{\log \left( 1 + \frac{z_{t,k}^2}{\sigma^2}   \right)}{\log \left( 1 + \frac{1}{\sigma^2}   \right)}$. \footnote{Notice that for any
$y \in [0,1]$, we have $y \leq \frac{\log \left( 1 + \frac{y}{\sigma^2}   \right)}{\log \left( 1 + \frac{1}{\sigma^2}   \right)} = h(y)$. 
To see it, notice that $h(y)$ is a strictly concave function, and $h(0)=0$ and $h(1)=1$.}
Hence we have
\[
\sum_{t=1}^{n} \sum_{k=1}^{\rnd{K}_t} z_{t,k}^2 
\leq \frac{1}{\log \left( 1 + \frac{1}{\sigma^2}   \right)}  \sum_{t=1}^n \sum_{k=1}^{\rnd{K}_t} \log \left( 1 + \frac{z_{t,k}^2}{\sigma^2}   \right)
\leq \frac{d K \log 
 \left[
 1+  \frac{nK}{d \sigma^2}
 \right]}{\log \left( 1 + \frac{1}{\sigma^2}   \right)}.
\]
Finally, from Cauchy-Schwarz inequality, we have that
\[
\sum_{t=1}^{n} \sum_{k=1}^{\rnd{K}_t} z_{t,k} \leq \sqrt{nK} \sqrt{\sum_{t=1}^{n} \sum_{k=1}^{\rnd{K}_t} z_{t,k}^2 }
\leq K \sqrt{
\frac{dn \log 
 \left[
 1+  \frac{nK}{d \sigma^2}
 \right]}{\log \left( 1 + \frac{1}{\sigma^2}   \right)}
}.
\]
\end{proof}

\subsection{Bound on $P(\bar{\cE})$}

\begin{lemma}
For any $\sigma>0$, any $\delta \in (0,1)$, and any
\[
c \geq  \frac{1}{\sigma}  \sqrt{d \log \left( 1 + \frac{nK}{d \sigma^2} \right)+ 2 \log \left (\frac{1}{\delta} \right) } + \|\theta^*
\|_2 ,
\]
we have $P(\bar{\cE}) \leq \delta$.

\end{lemma}

\begin{proof}
We start by defining some useful notations. For any $t=1,2,\dots$, any $k=1,2, \dots, \rnd{K}_t$, we define
\[
\rnd{\eta}_{t,k}=\rnd{w}_t (\rnd{a}^t_k) - \bar{w} (\rnd{a}^t_k).
\]
One key observation is that $\rnd{\eta}_{t,k}$'s form a Martingale difference sequence (MDS).\footnote{Notice that the notion of
``time" is indexed by the pair $(t,k)$, and follows the lexicographical order.} Moreover, since $\rnd{\eta}_{t,k}$'s are bounded in $[-1,1]$
and hence they are conditionally sub-Gaussian with constant $R=1$. We further define that
\begin{align}
\rnd{V}_t = &  \sigma^2 M_t = \sigma^2 I + \sum_{\tau=1}^t \sum_{k=1}^{\rnd{K}_{\tau}} x_{\rnd{a}^{\tau}_k} x_{\rnd{a}^{\tau}_k}^T 
\nonumber \\
\rnd{S}_t = &  \sum_{\tau=1}^t \sum_{k=1}^{\rnd{K}_{\tau}} x_{\rnd{a}^{\tau}_k} \rnd{\eta}_{t,k} 
= B_t - \sum_{\tau=1}^t \sum_{k=1}^{\rnd{K}_{\tau}} x_{\rnd{a}^{\tau}_k} \bar{w} (\rnd{a}^t_k)
=  B_t  - \left[  \sum_{\tau=1}^t \sum_{k=1}^{\rnd{K}_{\tau}} x_{\rnd{a}^{\tau}_k} x_{\rnd{a}^{\tau}_k}^T \right] \theta^*
\nonumber
\end{align}
As we will see later, we define $\rnd{V}_t $ and $\rnd{S}_t$ to use the ``self normalized bound" developed in 
\cite{abbasi-yadkori11improved} (see Algorithm 1 of \cite{abbasi-yadkori11improved}). Notice that
\[
M_t \bar{\theta}_t = \frac{1}{\sigma^2} B_t =  \frac{1}{\sigma^2} \rnd{S}_t + \frac{1}{\sigma^2}  \left[  \sum_{\tau=1}^t \sum_{k=1}^{\rnd{K}_{\tau}} x_{\rnd{a}^{\tau}_k} x_{\rnd{a}^{\tau}_k}^T \right] \theta^*
 = \frac{1}{\sigma^2} \rnd{S}_t +  \left[ M_t -I \right] \theta^*,
\]
where the last equality is based on the definition of $M_t$. Hence we have
\[
\bar{\theta}_t  - \theta^* =  M_t^{-1} \left[ \frac{1}{\sigma^2} \rnd{S}_t -\theta^* \right].
\]
Thus, for any $e \in E$, we have
\begin{align}
\left  |  \langle 
x_e, \bar{\theta}_t  - \theta^*
\rangle  \right  |
= & \left |
x_e^T
M_t^{-1} \left[ \frac{1}{\sigma^2} \rnd{S}_t -\theta^* \right]
\right | 
\leq \| x_e \|_{M_t^{-1}} 
\|
\frac{1}{\sigma^2} \rnd{S}_t -\theta^*
\|_{M_t^{-1}} \nonumber \\
 \leq &
\| x_e \|_{M_t^{-1}}  \left[
\|
\frac{1}{\sigma^2} \rnd{S}_t \|_{M_t^{-1}} + \|\theta^*
\|_{M_t^{-1}}
\right]
 , \nonumber
\end{align}
where the first inequality follows from the Cauchy-Schwarz inequality and the second inequality follows from the triangle
inequality. Notice that $\|\theta^*
\|_{M_t^{-1}} \leq \|\theta^*
\|_{M_0^{-1}} = \|\theta^*
\|_2$, and 
$\|
\frac{1}{\sigma^2} \rnd{S}_t \|_{M_t^{-1}}  = \frac{1}{\sigma} \| \rnd{S}_t \|_{\rnd{V}_t^{-1}} $ (since $M_t^{-1}= \sigma^2 \rnd{V}_t^{-1}$), so we have
\begin{equation}
\left  |  \langle 
x_e, \bar{\theta}_t  - \theta^*
\rangle  \right  | \leq
\| x_e \|_{M_t^{-1}}  \left[
\frac{1}{\sigma} \| \rnd{S}_t \|_{\rnd{V}_t^{-1}} + \|\theta^*
\|_2
\right ].
\label{eq:bound2}
\end{equation}
Notice that the above inequality always holds. We now provide a high-probability bound on $\| \rnd{S}_t \|_{\rnd{V}_t^{-1}} $
based on ``self normalized bound" proposed in \cite{abbasi-yadkori11improved}.
From Theorem 1 of \cite{abbasi-yadkori11improved}, we know that for any $\delta \in (0,1)$, with probability at least $1-\delta$, 
we have
\[
\| \rnd{S}_t \|_{\rnd{V}_t^{-1}}  \leq
\sqrt{
2 \log \left(
\frac{\det(\rnd{V}_t)^{1/2} \det(\rnd{V}_0)^{-1/2}}
{\delta}
\right)
} \quad
\forall t=0,1,\dots
\]
Notice that $ \det(\rnd{V}_0)= \det(\sigma^2 I)=\sigma^{2d}$. Moreover, from the trace-determinant inequality, we have
\[
\left[ \det (\rnd{V}_t) \right]^{1/d} \leq \frac{\tr \left ( \rnd{V}_t \right )}{d} = \sigma^2 + \frac{1}{d} \sum_{\tau=1}^t \sum_{k=1}^{\rnd{K}_{\tau}}
\| x_{\rnd{a}^t_k}\|_2^2 \leq \sigma^2 + \frac{tK}{d} \leq  \sigma^2 + \frac{nK}{d},
\]
where the second inequality follows from the assumption that $\| x_{\rnd{a}^t_k}\|_2 \leq 1$ and $\rnd{K}_{\tau} \leq K$, and the last inequality follows from $t \leq n$. Thus, with probability at least $1-\delta$, we have
\[
\| \rnd{S}_t \|_{\rnd{V}_t^{-1}}  \leq \sqrt{d \log \left( 1 + \frac{nK}{d \sigma^2} \right)+ 2 \log \left (\frac{1}{\delta} \right) } \quad
\forall t=0,1,\dots, n-1.
\]
That is, with probability at least $1-\delta$, we have
\begin{equation}
\left  |  \langle 
x_e, \bar{\theta}_t  - \theta^*
\rangle  \right  | \leq
\| x_e \|_{M_t^{-1}}  \left[
\frac{1}{\sigma}  \sqrt{d \log \left( 1 + \frac{nK}{d \sigma^2} \right)+ 2 \log \left (\frac{1}{\delta} \right) } + \|\theta^*
\|_2
\right ] \nonumber
\end{equation}
for all $t=0,1,\dots, n-1$ and $\forall e \in E$. 
Recall that by definition of event $\cE$, the above inequality implies that, if
\[
c \geq \frac{1}{\sigma}  \sqrt{d \log \left( 1 + \frac{nK}{d \sigma^2} \right)+ 2 \log \left (\frac{1}{\delta} \right) } + \|\theta^*
\|_2,
\]
then $P (\cE) \geq 1-\delta$. That is, $P (\bar{\cE}) \leq \delta$.
\end{proof}

\subsection{Conclude the Proof}
Putting it together, for any $\sigma>0$, any $\delta \in (0,1)$, and any
\[
c \geq  \frac{1}{\sigma}  \sqrt{d \log \left( 1 + \frac{nK}{d \sigma^2} \right)+ 2 \log \left (\frac{1}{\delta} \right) } + \|\theta^*
\|_2 ,
\]
we have that
\begin{align}
R(n) \leq  & 2c \E{
 \sum_{t=1}^n \sum_{k=1}^{\rnd{K}_t}     \sqrt{x_{\rnd{a}^{t}_k}^T M_{t-1}^{-1} x_{\rnd{a}^{t}_k}} 
  \middle | \cE}{} + nK  P(\bar{\cE}) \nonumber \\
 \leq & 2c K \sqrt{
\frac{dn \log 
 \left[
 1+  \frac{nK}{d \sigma^2}
 \right]}{\log \left( 1 + \frac{1}{\sigma^2}   \right)} }+ nK \delta.
\end{align}
Choose $\delta=\frac{1}{nK}$, we have the following result:
for any $\sigma>0$ and any 
\[
c \geq  \frac{1}{\sigma}  \sqrt{d \log \left( 1 + \frac{nK}{d \sigma^2} \right)+ 2 \log \left ( nK \right) } + \|\theta^*
\|_2 ,
\]
we have
\[
R(n) \leq 2c K \sqrt{
\frac{dn \log 
 \left[
 1+  \frac{nK}{d \sigma^2}
 \right]}{\log \left( 1 + \frac{1}{\sigma^2}   \right)} }+ 1.
\]

\end{document}